\documentclass[oribibl]{llncs}

\usepackage{haldefs}
\usepackage{algorithm2e}
\usepackage{graphicx}
\usepackage[title]{appendix}

\newcommand{\bbM}{\ensuremath{\mathbb M}}
\newcommand{\calI}{\ensuremath{\mathcal I}}
\newcommand{\M}[1]{\ensuremath{\mathbb M_{#1}}}
\newcommand{\I}[1]{\ensuremath{\mathcal I_{#1}}}
\newcommand{\hide}[1]{}

\usepackage{mathtools,amsmath,amssymb}

\usepackage{thmtools}
\usepackage{thm-restate}

\title{On Correcting Inputs: Inverse Optimization for Online
  Structured Prediction \thanks{Partially supported by NSF grant
     IIS-1451430.}
}

\author{Hal Daum\'e III
	\and Samir Khuller
        \and Manish Purohit
	\and Gregory Sanders}

\institute{University of Maryland, College Park, MD 20742, USA
        \\  {\tt
        	\{hal, samir, manishp, gsanders\}@cs.umd.edu} }

\begin{document}
\maketitle

\begin{abstract}
Algorithm designers typically assume that the input data is correct, and then proceed to find ``optimal'' or ``sub-optimal'' solutions using this input data.
However this assumption of correct data does not always hold in
practice, especially in the context of online learning systems where
the objective is to learn appropriate feature weights given some training samples. Such scenarios necessitate the study of inverse optimization problems where one is given an input instance as well as a desired output and the task is to adjust the input data so that the given output is indeed optimal. Motivated by learning structured prediction models, in this paper we consider inverse optimization with a margin, i.e., we require the given output to be better than all other feasible outputs by a desired margin. We consider such inverse optimization problems for maximum weight matroid basis, matroid intersection, perfect matchings, minimum cost maximum flows, and shortest paths
and derive the first known results for such problems with a
non-zero margin. 
The effectiveness of these algorithmic approaches
to online learning for structured prediction is also discussed.
%\keywords{Inverse optimization, structured prediction, online learning}
\end{abstract}

\section{Introduction}
Algorithm designers generally assume that the input data is sacrosanct
and correct.
Algorithms are then typically run on this input data to compute
``optimal'' or ``sub-optimal'' solutions quickly whether it be 
the computation of a maximum spanning tree, a maximum matching, 
max weight arborescence, or shortest paths.
However, with an increasing reliance on automatic methods to collect
data, as well as in systems that learn, this assumption does not always hold.
The input data can be erroneous (even though it may be approximately
correct), and it becomes important to ``adjust'' the input data to 
achieve certain desired conditions. 

A simple example can be used to illustrate the main point --  suppose we are 
given a weighted graph $G=(V,E)$ and a spanning tree $T$, and told that $T$ {\em should} be a maximum weight spanning
tree in $G$. The goal now is to perturb the edge weights of the graph
$G$, minimizing the $L_2$ norm of the perturbation, so that $T$ is indeed the optimal spanning tree. This kind of problem has been studied
previously in the form of ``Inverse Optimization'' problems. 
However, we wish to accomplish a stronger goal of making sure that
the given tree $T$  is better than {\em every} other tree in $G$
by a given margin $\delta$.

Our initial motivation for studying this problem comes from the \emph{structured prediction} task in machine learning \cite{lafferty01crf,punyakanok01inference,collins02perceptron,taskar05mmmn,tsochantaridis05svmiso}.
For concreteness and ease of exposition, we now describe structured prediction in the context of predicting dependency parse trees for natural language sentences. Given an English sentence, its dependency parse is a rooted, directed tree that indicates the dependencies between different words in the sentence as shown in Figure \ref{fig:dependency-parse}. The input sentence can be represented as a complete, directed graph on the words of the sentence that is parameterized by \emph{features} on the edges. Given a learned model (represented as a vector of parameters), the weight of an edge is computed as the inner product of its feature vector and the model. As linguistic constraints dictate that the required dependency parse must form a rooted, spanning arborescence of the graph, one can use off-the-shelf combinatorial algorithms~\cite{edmonds67mst,chu65mst} to find the highest weight arborescence. The learning problem is thus to find a parameter vector such that once the edges are weighted by the inner products, running a combinatorial optimization algorithm would return the desired parse tree. At ``training time'', we are given a sentence as well as its correct parse tree and the problem that we need to solve is exactly the inverse optimization problem - given the current model and the parse tree, say $T$, find the minimum perturbation to the model so that the combinatorial optimization algorithm would return $T$. It is well established in the learning theory literature that achieving a large margin solution enables better generalization~\cite{crammer03mira}. We consider minimizing the $L_2$ norm because of connections to prior work~\cite{kivenen95eg}. In particular, for applications in structured prediction, the convergence and error bounds (included in Section \ref{sec:application}) require $L_2$ norm minimization.

\begin{figure}
  \centering
  \includegraphics[width=0.6\textwidth]{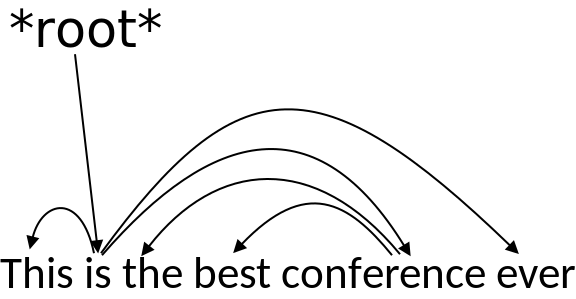}
\caption{Example dependency parse tree. The tree describes the relations between head words and their dependents in the sentence}
  \label{fig:dependency-parse}
\end{figure}

%   In
% this setting, the task is to learn a function that maps a given input
% to a discrete combinatorial structure. In most cases, the input can be
% viewed as a graph, and the desired output as a particular
% sub-structure on this graph (eg., tree, arborescence, matching, etc.). The graphs are parameterized by \emph{features} on the edges and the edge weights are computed as inner products between these feature vectors and a learned model (represented as a vector of parameters). The learning goal is to find a parameter vector such that once the edges are weighted by the inner products, running a combinatorial optimization algorithm would return the desired output structure. At ``training time'', these desired output structures are given and the problem we need to solve is exactly the inverse optimization problem - given the current model and the desired output structure, say $T$, find the minimum perturbation to the model so that the combinatorial optimization algorithm would return $T$. It is well established in the learning theory literature that achieving a large margin solution enables better generalization~\cite{crammer03mira}. We consider minimizing the $L_2$ norm because of connections to prior work~\cite{kivenen95eg}. In particular, for applications in structured prediction, the convergence and error bounds (included in Section \ref{sec:application}) require $L_2$ norm minimization.

In our work we consider such inverse optimization problems with
a margin in a general matroid setting. We consider both the problem
of modifying the weights of the elements of a matroid, so that a given
basis is a maximum weight basis (with a given margin of $\delta$) and the considerably harder problem of matroid intersection
where a given basis of two matroids should have weight higher (by at least $\delta$) than any other set of elements that is a basis in the two matroids.
This framework captures
two special cases which are useful for structured prediction - namely maximum weight bipartite matching (useful for language translation) and maximum weight arborescence (useful for sentence parsing). We also consider $\delta$-margin inverse optimization problems for a number of other classical combinatorial optimzation problems such as perfect matchings, minimum cost flows and shortest path trees. In addition, we present a generic framework for online learning for structured prediction using the corresponding inverse optimization problem as a subroutine and present convergence and error bounds on this framework.

\subsection{Related Work}
\label{sec:related-work}
Inverse optimization problems have been widely studied in the Operations Research literature. 
 Most prior work however has focused on minimizing the $L_1$ or $L_{\infty}$ norms between the weight vectors and, more importantly, do not allow non-zero margin ($\delta$).
Heuberger \cite{heuberger2004} provides an excellent survey of the diverse inverse optimization problems that have been tackled. 
Both the inverse matroid optimization~\cite{dellamico2003} and matroid intersection~\cite{mao1997} have previously been studied in the setting of minimizing the $L_1$ norm and with zero margin. However, they use techniques that are specialized to minimizing the $L_1$ norm of the perturbation and do not extend to minimizing the $L_2$ norm. At the same time, these approaches to do not generalize to the general case of inverse optimization with non-zero margins.

In typical global models for structured prediction (for e.g. see \cite{lafferty01crf,mcallester04cfd,tsochantaridis05svmiso,collins02perceptron,crammer06pa,mcdonald05online}), the discrete optimization problem is considered a ``black box''. By treating the combinatorial problem as a black box, these methods lose the ability to precisely reason about how certain changes to the underlying parameter vector can affect the eventual output. The simplest approach to solving the online structured prediction problem is the structured perceptron~\cite{collins02perceptron}. On each example, the
structured perceptron makes a prediction based on its current
model. If this prediction is incorrect, the algorithm suffers
unit loss and updates its parameters with a simple linear
update that moves the predictor closer to the truth and further from
the current best guess.  While empirically successful in a number of
problems, this particular update is relatively imprecise: there are
typically an exponential number of possible outputs for any given
input, and simply promoting the correct one and demoting the models'
current prediction may do very little to move the model as far as it
needs to go. An alternative approach is the large margin discriminative approach~\cite{crammer03mira} that seeks to change the parameters as little as possible subject to the constraint that the true output has a higher score than all incorrect outputs. However, such an approach is often computationally infeasible for structured prediction as there are usually an exponential number of potential outputs. McDonald et al.\ \cite{mcdonald05online} circumvent this infeasibility by using a $k$-best list of
possible outputs and restrict the set of constraints to require that the true output has a higher score than the incorrect outputs on the $k$-best list.  This has been shown to be effective for small values of $k$ on simple parsing tasks \cite{mcdonald05online}.
However, for more complex tasks, like machine translation, one
needs more complicated update frameworks \cite{chiang11hopefear}. In this work we show that the large margin discriminative approach is applicable to a wide range of problems in structured prediction using techniques from inverse combinatorial optimization.

In the context of online prediction, the most related work to ours is that of Taskar et
al.~\cite{taskar05mmmn}, who also consider structured prediction using
inverse bipartite matchings. They define a loss function that
measures, against a ground truth matching, the number of mispredicted
edges in the found matching. This ``Hamming distance'' style loss
function nicely decomposes over the structure of the graph and thereby
admits an efficient ``loss augmented'' inference solution, in which
correct edges are penalized during learning. (The idea is that if
correct edges are penalized, but the model still produces the correct
matching, then it has done so with a sufficiently large margin.) This
idea only works in the case of decomposable loss functions, or the
simpler 0-margin formulation. In comparison, our approach works both
for decomposable loss functions as well as ``zero/one loss'' over the
entire structure. Furthermore, our approach generalizes to arbitrary
matroid intersection problems and minimum cost flows and thus is applicable to a much wider range of structured prediction problems.

\subsection{Contribution and Techniques}
\label{sec:our-results}

A lot of prior work in the inverse optimization literature formulates the problem as a linear program and then uses strong duality conditions to find the new perturbed weights. However, such techniques cannot be extended to handle a non-zero margin that is required by the application to structured prediction.
We formulate inverse optimization to minimize the $L_2$ norm of the perturbations as a quadratic program and use problem specific optimality conditions to determine a concise set of linear constraints that are both necessary and sufficient to guarantee the required margin. In particular, one of the key ingredients is a set of polynomially many linear constraints that ensure that an appropriately defined auxiliary graph does not contain small directed cycles. We note that our formulations can easily be adapted to minimize the $L_1$ norm of the perturbations by simply modifying the objective and using linear programming.

We obtain concise formulations for exactly solving $\delta$-margin inverse optimization problems for (i) maximum weight matroid basis, (ii) maximum weight basis in the intersection of two matroids, (iii) shortest $s$-$t$ path, (iv) shortest path tree, (v) minimum cost maximum flow in a directed graph. %Most prior work formulates the inverse optimization problem as a linear program and use strong duality conditions to find the new weights. Such techniques are not suitable to solve inverse optimization problems with a non-zero $\delta$ margin.

% The motivating application to study inverse optimization with margins is for online structured prediction.
We also present convergence results for the generic online learning framework for structured prediction motivating our study.

The rest of the paper is organized as follows. In Section \ref{sec:problem-definition}, we formally define $\delta$-margin inverse optimization. In Sections \ref{sec:matroids} and \ref{sec:matroid-intersection}, we present our results on inverse optimization for matroids, and matroid intersections respectively. In Sections \ref{sec:perfect-matching}, Appendix \ref{sec:mincostflow}, and Appendix \ref{sec:sptrees}, we discuss inverse optimization for perfect matchings in bipartite graphs, minimum cost flows, and shortest path trees. In Section \ref{sec:application}, we describe an online learning framework for structured prediction as an application and the proof of convergence and error bounds for this learning framework are presented in Appendix \ref{sec:proofs-learn-theory}. Experimental results for our learning model are presented in Appendix \ref{sec:experiments} showing significant improvement over previous techniques.
% {\bf TODO: Manish - outline of paper with pointers to section numbers.
% Discuss experimental work and the new section Hal added.}

%%% Local Variables: 
%%% mode: latex
%%% TeX-master: "main"
%%% End: 

\section{Problem Description}
\label{sec:problem-definition}

As explained in the introduction, we require a given solution to be better than all other feasible solutions by a margin of $\delta$. We now formalize this notion of $\delta$-optimality.
\begin{definition}[$\delta$-Optimality]
  For a maximization problem $P$, let $\mathcal{F}$ denote the set of feasible solutions, let $w$ be the weight vector, $c(w,A)$ denote the cost of feasible solution $A$ under weights $w$, and let $\delta \geq 0$ be a scalar. A feasible solution $S \in \mathcal{F}$ is called $\delta$-optimal under weights $w$ if and only if 
\[c(w,S) \geq c(w,S') + \delta, \quad \forall S'(\neq S) \in \mathcal{F}.\]
\end{definition}

% \begin{definition}[$\delta$-Optimality]
%   For a maximization problem $P$, let $\mathcal{F}$ denote the set of feasible solutions, let $w:\mathcal{F} \rightarrow \mathbb{R}^+$ be the cost function, and let $\delta \geq 0$ be a scalar. A feasible solution $S \in \mathcal{F}$ is called $\delta$-optimal if and only if 
% \[c(S) \geq c(S') + \delta, \quad \forall S'(\neq S) \in \mathcal{F}.\]
% \end{definition}

$\delta$-optimality for minimization problems is defined similarly.
All problems we consider in this work can be classified as $\delta$-margin inverse optimization. 
\begin{definition}[$\delta$-Margin Inverse Optimization]
  For a given optimization problem $P$, let $\mathcal{F}$ denote the set of feasible solutions, let $w$ be the weight vector, let $\delta \geq 0$ be a scalar, and let $S \in \mathcal{F}$ be a given feasible solution. $\delta$-Margin Inverse optimization is to find a new weight vector $w'$ minimizing $||w' - w||_2$ ($L_2$ norm) such that $S$ is the $\delta$-optimal solution of $P$ under weights $w'$.
\end{definition}

% \begin{definition}[$\delta$-Margin Inverse Optimization]
%   For a given optimization problem $P$, let $\mathcal{F}$ denote the set of feasible solutions, let $w:\mathcal{F} \rightarrow \mathbb{R}^+$ be the cost function, let $\delta \geq 0$ be a scalar, and let $S \in \mathcal{F}$ be a given feasible solution. $\delta$-Margin Inverse optimization is to find a new weight function $w':\mathcal{F} \rightarrow \mathbb{R}^+$ minimizing $||w' - w||_2$ ($L_2$ norm) such that $S$ is the $\delta$-optimal solution of $P$.
% \end{definition}

In the following sections we consider $\delta$-margin inverse optimization for a number of problems mentioned earlier.
\hide{
We first recall the definition of a matroid.
\begin{definition}[Matroid]
A matroid is a pair $\bbM = (X, \calI)$ where $X$ is a ground set of elements and $\calI$ is a family of subsets of $X$ (called {Independent} sets) such that - 
\begin{enumerate}
\item $\calI \neq \phi$
\item (Hereditary) If $B \in \calI$, and $A \subseteq B$, then $A \in \calI$.
\item (Exchange property) If $A,B \in \calI$, and $|A| < |B|$, then there exists some element $e \in B \setminus A$ such that $A \cup \{e\} \in \calI$.
\end{enumerate}  
\end{definition}

In this paper, we consider the following problems.

\begin{definition}[Inverse Matroid Optimization]
Given a matroid $\bbM = (X,\calI)$, a weight function $w:X \rightarrow \mathbb{R}^+$, a basis $B$ of $\bbM$, and $\delta \geq 0$, the inverse matroid optimization problem is to find a new weight function $w':X \rightarrow \mathbb{R}^+$ minimizing $||w - w'||_2$ (L2 norm) such that $B$ is the $\delta$-optimal basis of $\bbM$ under $w'$.
\end{definition}

\begin{definition}[Inverse Matroid Intersection]
Given two matroids $\M{1} = (X,\I{1})$ and $\M{2} = (X,\I{2})$ defined on the same ground set $X$, a weight function $w:X \rightarrow \mathbb{R}^+$, a common basis $B$ of $\M{1}$ and $\M{2}$, and $\delta \geq 0$, the inverse matroid intersection problem is to find a new weight function $w':X \rightarrow \mathbb{R}^+$ minimizing $||w - w'||_2$ (L2 norm) such that $B$ is the $\delta$-optimum common basis of \M{1}~and \M{2} under $w'$.
\end{definition}
}

%%% Local Variables: 
%%% mode: latex
%%% TeX-master: "main"
%%% End: 

\section{Maximum weight matroid basis}
\label{sec:matroids}
In order to provide intuition about the type of problems we propose to solve in this paper, we first begin with the simple case of Inverse Matroid Optimization. We recall the definition of a matroid.
\begin{definition}[Matroid]
A matroid is a pair $\bbM = (X, \calI)$ where $X$ is a ground set of elements and $\calI$ is a family of subsets of $X$ (called {Independent} sets) such that - 
(i) $\calI \neq \phi$
(ii) (Hereditary) If $B \in \calI$, and $A \subseteq B$, then $A \in \calI$.
(iii) (Exchange property) If $A,B \in \calI$, and $|A| < |B|$, then there exists some element $e \in B \setminus A$ such that $A \cup \{e\} \in \calI$.
\end{definition}

\begin{definition}[Matroid Basis and Circuit]
  Let $\bbM = (X, \calI)$ be a matroid. Then any maximal independent set in $\calI$ is called a basis of the matroid. Conversely, any minimal dependent set is called a circuit.
\end{definition}
%Definitions of other matroid concepts such as basis and circuits are standard and can be found in textbooks (See e.g. \cite{schrijver2003}).

For the inverse problem we are given a matroid $\bbM = (X, \calI)$, a
weight function $w$ on the elements, and a basis $B$ of \bbM. The goal
is to find a weight function $w'$ so that $B$ is the $\delta$-optimal
basis of \bbM ~under the new weights. As it is well known that a
spanning tree is a basis of a graphical matroid, this inverse matroid
optimization problem directly generalizes the inverse maximum spanning
tree problem.

We first state a simple optimality condition for a given basis $B$ of a matroid $\bbM$. 
An easy generalization of \cite{schrijver2003} for $\delta \geq 0$ gives the following lemma.

\begin{lemma} \label{lem:matroids-1}
A given basis $B$ of a matroid $\bbM$ is $\delta$-optimal (under weight function $w$) if and only if for any $f \notin B$, and each $e \in C_B(f)$, $w(e) - w(f) \geq \delta$, where $C_B(f)$ denotes the unique circuit in $B \cup \{f\}$.
\end{lemma}
\hide{
\begin{proof}
 {\bf Necessity} (Easy Direction): If the condition for any $e,f$ pair is not satisfied, then $B' = B - \{e\} \cup \{f\}$ is also independent in $\bbM$ and $w(B) - w(B') = w(e) - w(f) < \delta$, which implies that $B$ is not $\delta$-optimal.

 {\bf Sufficiency}: Suppose that $w(e) - w(f) \geq \delta$ for all $e,f$ pairs. Let us assume that $B$ is not $\delta$-optimal. Let $B'$ be a basis such that $w(B) - w(B') < \delta$ and further among all such bases $B'$ is the one closest to $B$, i.e. $|B' \setminus B|$ is minimum. Consider $f^* \in B' \setminus B$. By the matroid exchange property, there exists $e^* \in B \setminus B'$ such that $e^* \in C_B(f^*)$ and $B'' = B' - \{f^*\} \cup \{e^*\}$ is also a basis. Hence,
 \begin{align}
   \label{eq:4}
   w(B'') &= w(B') + w(e^*) - w(f^*) \geq w(B') + \delta \\
\therefore   w(B) - w(B'') &\leq w(B) - w(B') < \delta    
 \end{align}
Since $|B'' \setminus B| < |B' \setminus B|$ we have a contradiction. \qed
\end{proof}
}
We thus have a set of polynomially many linear constraints that are necessary and sufficient for the given basis $B$ to be $\delta$-optimal. The inverse matroid optimization problem can then be formulated as a linearly constrained quadratic problem as follows - 
\begin{align} 
& 
\min_{w'} \sum_{e \in X}(w'(e) - w(e))^2
\quad\textbf{subj. to:}\quad \label{eq:1}\\ 
 & w'(e) - w'(f) \geq \delta,\quad
  \forall f \notin B, \forall e \in C_B(f) \label{eq:2}
\end{align}

Such a program with a quadratic objective and linear constraints can be solved in polynomial time and a number of practical solvers such as \cite{gurobi} are available.

%%% Local Variables: 
%%% mode: latex
%%% TeX-master: "main"
%%% End: 

\section{Matroid Intersection}
\label{sec:matroid-intersection}

Similar to the case with a single matroid, we need to derive a necessary and sufficient condition for a common basis $B$ of two matroids to be $\delta$-optimal. We can establish such an optimality condition with the help of an exchange graph associated with the basis $B$ and matroids $\M{1}$ and $\M{2}$.
\begin{definition}[Exchange Graph]
Given two matroids $\M{1} = (X,\I{1})$ and $\M{2}=(X,\I{2})$, a weight function $w:X \rightarrow \mathbb{R}^+$, and a common basis $B$, an \emph{exchange} graph is a directed, bipartite graph $G=(V,A)$ with a length function $l$ on edges that is defined as follows.
\begin{align}
  \label{eq:3}
  V &= B \cup X \setminus B\\
  A &= A_1 \cup A_2\\
  A_1 &= \{(x,y) | x \in B, y \in X \setminus B, B - \{x\} + \{y\} \in \I{1}\}\\
  A_2 &= \{(y,x) | x \in B, y \in X \setminus B, B - \{x\} + \{y\} \in \I{2}\}\\
  l(s) &=
  \begin{cases}
    w(x) &\text{if $s=(x,y) \in A_1$}\\
    -w(y) &\text{if $s=(y,x) \in A_2$}
  \end{cases}
\end{align}
\end{definition}

The above graph captures the exchange operations that can be performed. An edge $(e,f)$ implies that deleting $e$ and adding $f$ to $B$ preserves independence w.r.t matroid \M{1} and similarly for the other direction. As the graph is bipartite, every cycle is of even length - a cycle $C = (x_1 , y_1, x_2, y_2, \ldots x_k, y_k, x_1)$ corresponds to constructing a set $B'' = B - \{x_1, x_2, \ldots x_k\} \cup \{y_1, y_2, \ldots, y_k\}$. Further 
\[w(B'') = w(B) - \sum_{i=1}^k w(x_i) + \sum_{i=1}^k w(y_i) = w(B) - l(C)\] where $l(C) = \sum_{e \in C} l(e)$ is the sum of lengths of edges in the cycle $C$.
We are now in a position to present the $\delta$-optimality condition of $B$ in terms of the exchange graph. Fujishige~\cite{fujishige1977a} shows the following lemma for the case of $\delta = 0$. We include the extended proof for general $\delta$ margin here for completeness. It is important to note that while there are other optimality conditions for matroid intersection such as the weight decomposition theorem by Frank \cite{frank1981weighted}, these conditions do not easily generalize for non-zero $\delta$.
%An easy generalization of
%\cite{schrijver2003,
%\cite{fujishige1977a} proves the following lemma %\footnote{Note that other optimality conditions for matroid intersection such as weight decomposition by Frank \cite{frank1981weighted} does not generalize for non-zero $\delta$}. 

\begin{lemma}[Matroid Intersection $\delta$-optimality condition]
\label{lemma:optimality-condition}
The given common basis $B$ is $\delta$-optimal if and only if the exchange graph G contains no directed cycle $C$ such that $\sum_{e \in C} l(e) \leq \delta$.
\end{lemma}

%%STARTS HERE 

\hide{
We'll refer to two well-known lemmas regarding the relationship between bases of a matroid and matchings in the exchange graph. Let $G(B,B')$ denote the subgraph induced on the $G$ by the vertex sets $B \setminus B'$ and $B' \setminus B$.
\begin{lemma}
\label{lemma:perfect-matching}
  If $B'$ is a basis of  matroid \M{1} $[\M{2}]$, then $G_1(B,B')[G_2(B,B')]$ contains a perfect matching.
\end{lemma}

\begin{lemma}
\label{lemma:unique-perfect-matching}
  For $B' \subseteq S$, if $G_1(B,B') [G_2(B,B')]$ has a \emph{unique} perfect matching, then $B'$ is a basis of \M{1} $[\M{2}]$.
\end{lemma}

Observe the asymmetry in the above two lemmas. The requirement of a \emph{unique} perfect matching causes the following proof to be a bit more involved.
}
\begin{proof}
We'll refer to two well-known lemmas~\cite{schrijver2003} regarding the relationship between bases of a matroid and matchings in the exchange graph. Let $G_1 = (V,A_1)$ and $G_2 = (V,A_2)$ be the subgraphs of $G$ induced by the two matroids respectively. Further for $B' \subset X$, let $G(B,B')$ denote the subgraph induced on the $G$ by the vertex sets $B \setminus B'$ and $B' \setminus B$.
\begin{lemma}
\label{lemma:perfect-matching}
  If $B'$ is a basis of  matroid \M{1} $[\M{2}]$, then $G_1(B,B')[G_2(B,B')]$ contains a perfect matching. $\hfill\qed$
\end{lemma}

\begin{lemma}
\label{lemma:unique-perfect-matching}
  For $B' \subseteq X$, if $G_1(B,B') [G_2(B,B')]$ has a \emph{unique} perfect matching, then $B'$ is a basis of \M{1} $[\M{2}]$. $\hfill\qed$
\end{lemma}

{\bf Sufficiency:}  
This is the easy direction.
Let $B'$ be any common basis other than $B$. Applying Lemma \ref{lemma:perfect-matching}, we know that $G(B,B')$ has two perfect matchings (one each in $G_1(B,B')$ and $G_2(B,B')$). Union of these two perfect matchings yields a collection of cycles $\mathcal{C}$. Further, by construction, by traversing these cycles, one can transform $B \rightarrow B'$ and hence, we have
$w(B') = w(B) - \sum_{C \in \mathcal{C}}l(C)$.
Therefore, since we have $l(C) > \delta$ for all cycles, we are guaranteed that $w(B') < w(B) - \delta$ as desired.

{\bf Necessity:}
Ideally, we would like to say that every cycle in $G$ leads to a swapping such that the set so obtained is also independent in both the matroids. This would immediately imply that a cycle of small length would lead to a common basis $B'$ which is not much smaller than $B$.

However, the presence of a cycle simply implies the presence of a perfect matching (one in each direction) which may not be \emph{unique}. For example, Figure \ref{fig:original} shows an instance of an arborescence problem (left), and the associated exchange graph (right). Here $G$ contains a cycle a-x-b-y-a which leads to a new set ${x,y,c}$ which is not an arborescence. 

\begin{figure}[htbp]
 \centering
 \includegraphics[width=0.5\textwidth]{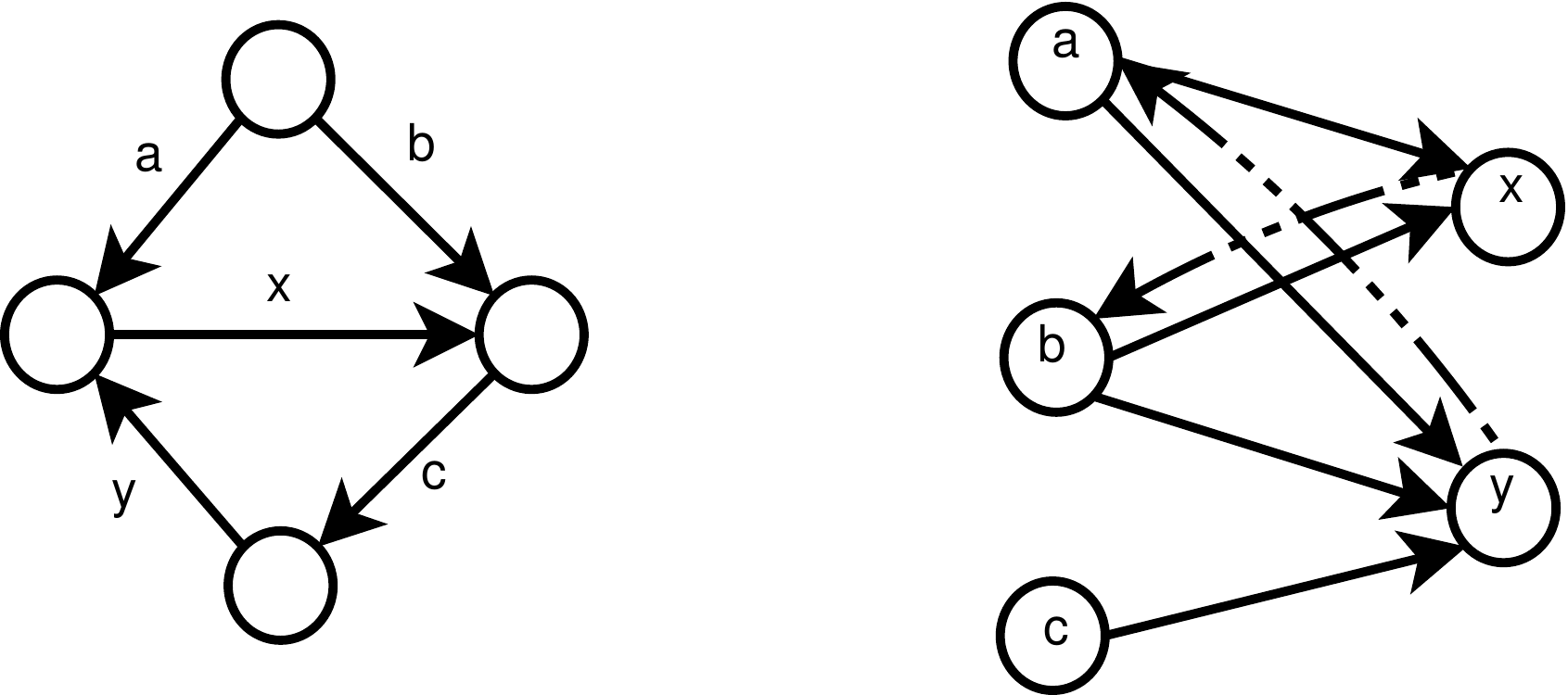}
 \caption{Instance showing every cycle in $G$ need not lead to a common basis}
 \label{fig:original}
\end{figure}

In the previous example, observe that if the cycle a-x-b-y-a were to have small weight, that would imply that at least one of a-y-a or b-x-b cycles too has small weight both of which lead to a feasible solution. This observation motivates us to look at the \emph{smallest} cycle of weight less than $\delta$ and hope that it does induce an \emph{unique} perfect matching.

Suppose that the graph has a cycle having weight less than $\delta$. Let $C$ be the smallest (in terms of number of arcs) such cycle. Look at the graph induced by the vertex set of the cycle. We claim that this induced subgraph has a unique perfect matching (one in each direction). Here we prove the claim for one direction. $C$ being an even cycle trivially contains a perfect matching $M$ from $B$-side to $X \setminus B$-side. Suppose there exists another perfect matching $M'$. For every edge $(x,y)$ in $M' \setminus M$, the edge along with the path between $y$ and $x$ in $C$ cause a cycle. Further, each such cycle is smaller (number of edges) than $C$. 

Let $\bar{M}$ denote the matching $M$ with edge directions reversed. The union of $M'$ and $\bar{M}$ now forms a collection of cycles. Consider any such cycle $D$. WLOG let the cycle be $(x_0, y_0, x_1, y_1, \ldots, x_k, y_k, x_0)$ such that the $(x_{i+1}, y_i)$ are edges in $M$ (i.e. $(y_i, x_{i+1}) \in \bar{M}$) and $(x_i,y_i) \in M'$. [All arithmetic is modulo $k+1$]. We'll now be interested in the length of the path between these vertices in the original cycle $C$. Let $C_i$ denote the cycle formed by the edge $(x_i,y_i)$ and the path between $y_i$ and $x_i$ in $C$. We have, 
\[l(C_i) = l(C) - l(\text{Path from $x_i$ to $y_i$ in $C$}) + l((x_i,y_i))\]
Since $(x_i, y_{i-1}) \in M$, 
\[l(\text{Path from $x_i$ to $y_i$ in $C$}) = l((x_i,y_{i-1})) + l(\text{Path from $y_{i-1}$ to $y_i$ in $C$})\]
Further since by construction $l((x,y_i)) = l((x,y_j)) ( = \pm w(x))$, we have
\[l(C_i) = l(C) - l(\text{Path from $y_{i-1}$ to $y_i$ in $C$})\]
Let $P_{i-1 \rightarrow i}$ denote this path. Summing over all $(x_i,y_i)$ edges in $D$, we get
\begin{align*}
 \sum_{i=0}^k l(C_i) &= k l(C) - (l(P_{k \rightarrow 0}) + l(P_{0 \rightarrow 1}) + \ldots + l(P_{k-1 \rightarrow k})) \\
 &= k l(C) - k' l(C) \\
\intertext{ $\uparrow$ Since we start from $y_k$, go around the $C$ and reach $y_k$ back}
 &= k'' l(C)\\
 &< k'' \delta
\intertext{The sum of $k$ weights is less than $k'' \delta$ with $k'' < k$, which implies}
\exists &C_i, \text{such that } l(C_i) < \delta
\end{align*}

But this is a contradiction since $C$ was the smallest cycle having weight less than $\delta$. Hence, the perfect matching $M$ is unique.
Similarly, the perfect matching induced by $C$ in the other direction too is unique. Applying Lemma \ref{lemma:unique-perfect-matching} successively on both sides, we know that $B'$ obtained by exchanging as per $C$ is a common basis for both matroids. Further, we have
\begin{align*}
 w(B') &= w(B) - l(C) \\
 w(B') &> w(B) - \delta
\end{align*}

Hence we have proved that if $G$ has a cycle with small weight, then $B$ is not $\delta$-optimal, thus proving the necessity of the claim.
\end{proof}

%% ENDS HERE
\subsection{Lower bounding cycles}
\label{sec:lower-bound-cycles}

In order to use Lemma \ref{lemma:optimality-condition} to solve the inverse matroid intersection problem efficiently using quadratic programming, we need a way to formulate this condition as a polynomial number of linear constraints. We now explore a technique to express the condition that a given graph has no small (of length less than $\delta$) cycles concisely.
%  Hence, essentially we now need a technique to express the condition that a given graph has no cycles having length less than $\delta$ as a polynomial number of linear constraints.
% We now explore a technique to lower bound the weight of all cycles.
Say we are given a directed graph $G = (V,A)$ and our task is to assign edge-lengths so that all cycles in $G$ have weight at least $\delta$. Letting the edge-lengths to be variables, the feasible region in this case is unbounded and is defined by a constraint for every cycle in $G$, i.e. we have the region $R_1$ in $m$ dimensions defined by - 
\begin{align}
	{\bf R_1:}&& \nonumber\\
  \sum_{e \in C} l_e &\geq \delta &\text{For all cycles C}
\end{align}
Of course, this formulation has an exponential number of constraints. Although the ellipsoid algorithm can be used to solve the quadratic program in polynomial time, it is often too slow for practical use. We now show that we can obtain a concise extended formulation by adding a few extra variables.

Suppose we have variables $d_{xy}$ representing the shortest distance between vertices $x$ and $y$. In this case, the graph has no cycle of weight less than $\delta$ if and only if $d_{xx} \geq \delta$ for all vertices $x$ (assume $d_{xx} = \infty$, if $x$ is not in any cycle). Consider the region $R_2$ in $m+n^2$ dimensions. %where the first two sets of constraints represent the triangle inequality constraints, 
\begin{align}
	{\bf R_2:}&& \nonumber\\
  d_{xy} &\leq l_{(xy)} &\text{For all $(x,y) \in A$} \label{r2edge}\\
  d_{xz} &\leq d_{xy} + l_{(yz)} &\text{For all $x,z \in V$ and $y$ s.t. $(y,z) \in A$} \label{r2triangle}\\
  d_{xx} &\geq \delta &\text{For all $x \in V$} \label{r2delta}
\end{align}
Constraints \eqref{r2edge} and \eqref{r2triangle} enforce triangle inequality, and \eqref{r2delta} enforce the condition that all cycles are large. We now prove that optimizing any function of $l$ over $R_1$ is equivalent to optimizing the same over $R_2$.

\begin{lemma}
\label{lemma:delta-cycles}
  $R_1$ is identical to the projection of $R_2$ on the $m$ dimensions corresponding to the edge-lengths. % Hence, optimizing any function of $l$ over $R_1$ is equivalent to optimizing the same over $R_2$.
\end{lemma}
\begin{proof}
{ $\mathbf {R_1 \subseteq \text{\bf Projection}(R_2)}$: }
  Let $l : E \rightarrow \mathbb{R}$ denote a point in $R_1$. Since the constraints \eqref{r2edge} and \eqref{r2triangle} are always valid for a \emph{true} distance function, let $d : V \times V \rightarrow \mathbb{R}$ denote the actual distance function in the graph induced by $l$. Such a $d$ definitely satisfies constraints \eqref{r2edge} and \eqref{r2triangle}. Additionally, for all vertices $x$ belonging to some cycle, since all cycles under $l$ have weight at least $\delta$, we have $d_{xx} \geq \delta$. For a vertex $x$ which does not belong to any cycle, one can simply set $d_{xx} = \infty$.

{\bf Projection$\mathbf{(R_2) \subseteq R_1}$:}
Consider a point in $R_2$. We now have the lengths of edges $l_e$ as well as some $d_{xy}$ values. Consider any cycle $C = (x_1, x_2, \ldots, x_k, x_1)$ in the graph. Applying constraint \eqref{r2triangle} repeatedly we get
\begin{align}
d_{x_1x_1} &\leq l_{(x_1x_2)} + l_{(x_2x_3)} + \ldots + l_{(x_{k-1}x_k)} + l_{(x_kx_1)}
\intertext{and also by constraint \eqref{r2delta}, we have}
  d_{x_1x_1} &\geq \delta
\end{align}
Hence we have, $l_{(x_1x_2)} + l_{(x_2x_3)} + \ldots + l_{(x_{k-1}x_k)} + l_{(x_kx_1)} \geq \delta$, i.e. $\sum_{e \in C} l_e \geq \delta$ which means that the $l_e$ values are feasible in $R_1$. %\qed
\end{proof}

Hence, optimizing any function of the $l_e$ variables over $R_1$ is equivalent to optimizing it over $R_2$. However, $R_2$ has only $m + mn + n$ constraints and $n^2 + m$ variables.

\subsection{Putting it together}
\label{sec:putting-it-together}

Lemmas \ref{lemma:optimality-condition} and \ref{lemma:delta-cycles} suggest a way to solve the $\delta$-margin inverse matroid intersection problem. As per the requirements of Lemma \ref{lemma:optimality-condition}, given the two matroids and the common basis $B$, construct the exchange graph $G=(V,A = A_1 \cup A_2)$. Let $w:X \rightarrow \mathbb{R}^+$ be the original weight function and let $w'$ be the new weight function which we desire. If $l$ is the arc lengths of $G$, according to the construction of Lemma \ref{lemma:optimality-condition}, $l_{xy} = w'(x)$ and $l_{yx} = -w'(y)$ where $x \in B, y \in S \setminus B$. Further, the objective that we minimize is the $L_2$ norm of $w - w'$. We can now add these additional constraints and the objective to the region $R_2$ as per Lemma \ref{lemma:delta-cycles} to obtain the minimum change on the weights of elements so that the exchange graph has no small cycles and hence $B$ is $\delta-$optimal. 
\begin{align} 
& 
\min_{w'} \sum_{e \in X}(w'(e) - w(e))^2
\quad&\textbf{subj. to:}\quad\\
 & l_{xy} = w'(x),  &\forall (x,y) \in A_1\\
 & l_{yx} = -w'(y),  &\forall (y,x) \in A_2\\
 &d_{xy} \leq l_{xy},  &\forall (x,y) \in A\\
 &d_{xz} \leq d_{xy} + l_{yz},  &\forall x,z \in V, \forall (y,z) \in A\\
 &d_{xx} \geq \delta,  &\forall x \in V
\end{align}

\subsection{Maximum Weight Arborescence}
\label{sec:arborescence}
Given a directed graph, a $r$-arborescence (also known as a branching) is the directed analogue of a spanning tree and is defined as a set of edges $T$ spanning all vertices such that every vertex (except $r$) has exactly one incoming edge in $T$. It is well known that an arborescence in a directed graph is a basis in the intersection of a graphical matroid and a partition matroid.  We analyze the complexity of the above technique for the special case of maximum weight arborescence. Let $G$ denote the graph in question having $n$ vertices and $m$ edges.
%{\bf Size of exchange graph:}
%\label{sec:size-exchange-graph}

The exchange graph $G_{ex}$ has a vertex for every edge of $G$, i.e., $n_{ex} = m$. The bipartition of $G_{ex}$ is such that we have components of size $n$ and $m-n$ respectively. Hence we have $m_{ex} = O(mn)$.
%{\bf Size of program:}
%\label{sec:size-program}
As seen in Section \ref{sec:lower-bound-cycles}, we use $O(n_{ex}^2)$ variables and $O(m_{ex}n_{ex})$ contraints. Thus, putting it all together, we have a quadratic program with $O(m^2)$ variables and $O(m^2n)$ constraints.

The inverse maximum weight arborescence problem is important as it can used as a subroutine in the online learning for dependency parsing~\cite{mcdonald05dependency}. The dependency parse tree of a sentence can be represented as an arborescence over a graph consisting of every word in the sentence as a node. In Appendix \ref{sec:experiments}, we show experimental results for dependency parsing using our framework.

\subsubsection{Shortest $s-t$ paths.}
Given a weighted graph $G=(V,E,w)$, a path $P$ between terminals $s$ and $t$, and a margin $\delta$, the inverse shortest $s$-$t$ path problem is to find a minimum perturbation to $w$ (minimizing the $L_2$ norm) so that $P$ is shorter than all other paths between $s$ and $t$ by at least $\delta$ under the new weight function. As shown by \hide{Hu and Liu~}\cite{zhiquan1998}, the inverse shortest $s$-$t$ path problem can be reduced to the inverse arborescence problem. Let $G'$ be $G$ augmented by adding zero weight edges from $t$ to all other vertices. It can be easily observed that $P$ is the shortest $s$-$t$ path in $G$ if and only if $P$ and a subset of the zero weight edges form the minimum weight $s$-arborescence of $G'$. Thus we can use an algorithm for inverse minimum weight arborescence to solve the inverse shortest path problem.\footnote{Inverse minimum weight arborescence problem can be solved similar to the inverse maximum weight arborescence problem} 

\section{Perfect Matchings in Bipartite Graphs}
\label{sec:perfect-matching}

For the bipartite maximum weight perfect matching inverse problem, the previous technique yields a quadratic program having $O(m^2)$ variables and $O(m^2)$ constraints as the exchange graph is sparse. In this section we show that we can in fact obtain more concise formulations. Recall that for a given edge weighted, bipartite graph $G = (X \cup Y,E,w)$, and a perfect matching $M$, an alternating cycle is a cycle in $G$ in which edges alternate between those that belong to $M$ and those that do not. An alternating cycle $C$ is called $\delta$-augmenting, if $\sum_{e \in C \cap M} w(e) < \sum_{e \in C \setminus M} w(e) + \delta$. The following characterization of a $\delta$-optimal perfect matching is well known.

\begin{lemma} \label{lem:perfect-matching}
A perfect matching $M$ is $\delta$-optimal if and only if the graph contains no $\delta$-augmenting cycles.
\end{lemma}

% We now define an auxiliary graph to capture the above optimality condition in terms of small directed cycles. Given an undirected graph $G = (V,E)$, a weight function $w$ on edges, and a perfect matching $M$, let $H = (V,A)$ be a directed graph such that $(x,z) \in A$ if and only if $\exists y \in V,$ such that $(x,y) \in M$ and $(y,z) \in E$, further let $l(x,z) = w(x,y) - w(y,z)$ (well defined as for a given $x$, $M$ uniquely determines $y$).

The central idea is to construct a directed graph $H$ on just the nodes of $X$ such that any directed cycle in $H$ will correspond to an alternating cycle in $G$ (w.r.t to the matching $M$) and vice versa. We construct $H = (X,A)$ to be a directed graph such that $(x,z) \in A$ if and only if $\exists y \in Y$ such that $(x,y) \in M$ and $(y,z) \in E$; further let $l(x,z) = w(x,y) - w(y,z)$. Figure \ref{fig:example} shows an example of this construction.

\begin{figure}[htbp]
  \centering
  \includegraphics[scale=0.6]{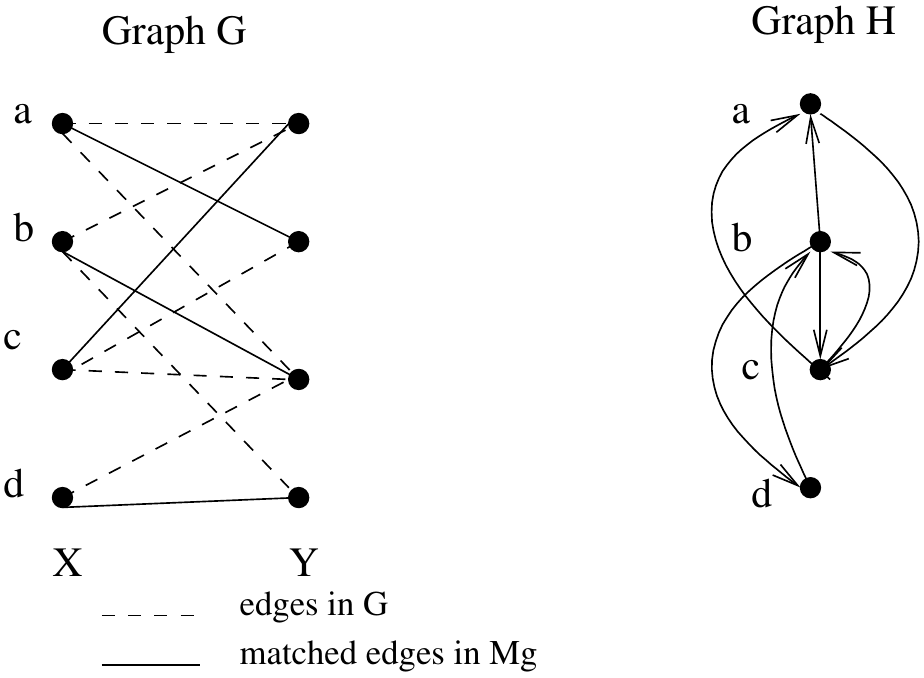}
  \caption{Example to show construction of $H$ from a bipartite graph $G$ and matching $M$}
  \label{fig:example}
\end{figure}
\begin{proposition} \label{prop:perfect-matching}
The auxiliary graph $H$ has a directed cycle of length less than $\delta$ if and only if $G$ has a $\delta$-augmenting alternating cycle.
\end{proposition}
\begin{proof}

{\bf If:} Let $C = (x_0,y_0,x_1,y_1,\ldots,x_k,y_k,x_0)$ be a $\delta$-augmenting cycle in $G$ where all $(x_i,y_i) \in M$. By construction, $H$ has a cycle $C' = (x_0,x_1,\ldots,x_k,x_0)$ and $l(C') = \sum_{i=0}^k (w(x_i,y_i) - w(y_i,x_{i+1}))$ (modulo $k+1$) $= \sum_{e \in C \cap M} w(e) - \sum_{e \in C \setminus M} w(e) < \delta$.

{\bf Only If:} Let $C = (x_0,x_1,\ldots,x_k,x_0)$ be a cycle in $H$ with $l(C) < \delta$. By construction, $\exists$ cycle $C' = (x_0,y_0,x_1,y_1,\ldots,x_k,y_k,x_0)$ in $G$. Now, $l(C) = \sum_{i=0}^k (w(x_i,y_i) - w(y_i,x_{i+1}))$ (modulo $k+1$) $= \sum_{e \in C' \cap M} w(e) - \sum_{e \in C' \setminus M} w(e)$. Thus $C'$ is a $\delta$-augmenting cycle in $G$. %\qed
\end{proof}

Using Lemma \ref{lem:perfect-matching} and Proposition \ref{prop:perfect-matching} along with Lemma \ref{lemma:delta-cycles}, we can formulate the inverse perfect matching problem as a quadratic program having $O(n^2)$ variables and $O(mn)$ constraints.%  - 

% \begin{align} 
% & 
% \min_{w'} \sum_{e \in E}   (w'(e) - w(e))^2
% \quad\textbf{subj. to:}\quad \\ 
% & l_{xz} = w'(x,y) - w'(y,z) &\forall (x,z) \in A, \text{ where } \\ \nonumber
% & &(x,y) \in M, \text{ and } (y,z) \in E \setminus M\\ 
%  & d_{xz} \leq l_{xz} & \forall (x,z) \in A\\
%  & d_{z'z} \leq d_{z'y} + l_{xz} &\forall z',z \in V, \forall (x,z) \in A\\
%  & d_{xx} \geq \delta &\forall x \in V 
% \end{align}

% This formulation has only $O(n^2)$ variables and $O(mn)$ constraints.
%%% Local Variables: 
%%% mode: latex
%%% TeX-master: "main"
%%% End: 

\section{Application : Online learning for structured prediction}
\label{sec:application}

\newcommand{\vDelta}{\vec\Delta}
\newcommand{\vphi}{\vec\phi}

In this section, we present a framework for online learning using inverse combinatorial optimization.
The structured prediction task is to predict a discrete combinatorial structure (such as an arborescence) given a structured input (such as a graph). The learning task is to learn model parameters so that solving a combinatorial optimization problem on the input instance would return the desired output structure. Structured prediction is extensively used in natural language processing tasks such as obtaining parse trees of a sentence, or automatic language translation. 
% The initial idea for using inverse combinatorial optimization for
% structured prediction is due to Taskar et al. \cite{taskar05mmmn} and follow up
% work \cite{taskar06sp}, in a ``batch'' learning setting
% for a limited set of problems (most notably: matchings).

In the online learning setting, we are presented with a set of $T$
training samples. These consist of an input $x_t$ (for instance, a sentence)
and an output $y_t$ (for instance, a syntactic analysis of this
sentence described as an arborescence on a graph over the words in the
sentence \cite{zeman1998statistical,mcdonald05dependency}).  Each edge in
this graph is parameterized by a set of $F$ features that, for
instance, indicate how likely one word is to be the subject of
another.
Thus, each training sample is a pair $(x_t, y_t)$ where $x_t$ is a
graph parameterized by features on edges, and $y_t$ is the desired
output sub-structure (such as a spanning tree, or an arborescence, or
a matching depending on the application). The task is to learn a
vector (of length $F$) of parameters $\vec{\theta}$ such that when
edge weights are computed as inner products between the $\vec{\theta}$
and the edge's features, the output obtained by computing an optimal
sub-structure (spanning tree, etc.) is the desired output with some
margin. 

Algorithm \ref{learningalgorithm} describes the generic online
learning framework for structured predcition. It is parameterized by an user-defined loss function
$\ell(y_t, \hat y)$ that specifies the loss incurred by the prediction $\hat y$ with respect to the training solution $y_t$. Algorithm \ref{learningalgorithm} is an adaptation of the Passive-Aggressive MIRA algorithm~\cite{koby-thesis} for structured prediction.

% zero if the training example is predicted
% correctly and one if any part of it is wrong.

%\vspace*{-0.5cm}
\begin{algorithm}
  $\vec{\theta}_1 = \vec{0}$ \\
    \For{$t = 1$ to $T$} {
      Obtain training example $x_t,y_t$\\
      $w \leftarrow$ weight function s.t. $w(e) = \vec{\theta}_t \cdot \vec{f_e}$ where $\vec{f_e}$ is feature vector of edge $e$\\
      $\hat{y} \leftarrow$ optimal sub-structure for graph $x_t$ under weights $w$ \\
      % Compute output $\hat y = \arg\min_y TODO$ -- somehow want to say
      % ``solve forward optimization problem on $x_t$''\\
      Suffer loss $\de_t = \ell(y_t, \hat y)$\\
      Update $\vec{\theta}_{t+1} = \argmin_{\vec{\theta}'} ||\vec{\theta}' - \vec{\theta}_{t}||_2^2 $ {\bf such that} \\
      \hspace*{5mm} $w' \leftarrow$ weight function s.t. $w'(e) = \vec{\theta}' \cdot \vec{f_e}$ where $\vec{f_e}$ is feature vector of edge $e$\\
      \hspace*{5mm} $y_t$ is the $\de_t$-optimal sub-structure for graph $x_t$ under weights $w'$
%      $\hat{\theta} = \hat{\theta} + \vec{\theta}_{i+1}$ \\
    }
  Return $\vec\theta_{T+1}$
\caption{Generic online learning framework}
\label{learningalgorithm}
\end{algorithm}
%\vspace*{-.5cm}
Note that the minimization problem solved for each training sample is
exactly $\delta$-inverse optimization where we minimize the
perturbations to the feature parameters instead of the edge
weights. In this framework, the different inverse
optimization problems we considered have applications for different
structured predictions. For example, maximum weight arborescences are
used to predict the parse tree of a sentence \cite{zeman1998statistical,mcdonald05dependency}, while maximum weight
matchings are used for language translation and word alignments~\cite{taskar05discrim}.

Since we have shown that we can efficiently solve the inverse optimization problems for a variety of combinatorial structures, we can extend the error bounds of the MIRA algorithm~\cite{koby-thesis} to work for learning the corresponding structured prediction models.
In this section, we present both convergence results and loss bounds
for our generic online learning framework.  
% In particular, we show
% that Algorithm~\ref{learningalgorithm} will converge after a finite
% number of steps (Theorem~\ref{thm:convergence}) and that, at
% convergence, it will only have made a small number of errors. That is,
% it's cumulative loss over the entire sequence $1\dots T$ will be
% bounded (Theorem~\ref{thm:thl}).
The proofs for these bounds closely follow those in Crammer's Ph.D. dissertation
\cite{koby-thesis}
and are relegated to 
Appendix \ref{sec:proofs-learn-theory} for clarity and brevity.

The statement of the convergence result depends on a set of dual
variables obtained from the optimization problem in the ``Update''
step of Algorithm~\ref{learningalgorithm}. This implicitly encodes
constraints over all possible outputs; we denote the dual variable for output $y$ on the $t^{\text{th}}$ example by $\al_y^t$.
% we denote by $\al_y^t$ (on the $t^{\text{th}}$ example) the dual variable for each
% possible output we denote $\al_y^t$ (on the $t^{\text{th}}$ example).
% In order to make progress this, we must first derive the dual of
% Eq~\eqref{eq:osp}.  This is a routine exercise, so we simple state the
% result:
% \begin{equation} 
% \cL^t(\vec\al) =
% -\frac 1 2
%    \sum_{y,z \in \cY^t}
%      \al_y^t \al_z^t \dotp{\vDelta^t_y}{\vDelta^t_z}
% +  \sum_{y \in \cY^t}
%      \al^t_y \left[ \de - \dotp{\vth^t}{\vDelta^t_y} \right]
% \label{eq:dual}
% \end{equation}
% %
% \noindent
% where $\vth^t$ is the previous weight vector, $\de$ is the desired
% margin, $\cY^t$ is the space of all possible structures for $x^t$ and
% $\vDelta^t_y$ is short-hand for $\vphi(x^t,y^t) - \vphi(x^t,y)$.
% This dual is maximized over $\vec\al$ subject to a
% non-negativity constraint.
We can show that the cumulative sum of these dual variables is bounded
by a constant independent of $T$, which implies convergence of the
learning algorithm.

\begin{restatable}[Convergence]{thm}{convergence} \label{thm:convergence}
  Let $\{ (x_t, y_t) \}_{t=1}^T$ be a sequence of structured examples.
Let $\vth^*$ be any vector that separates the data with a positive
margin $\de^* > 0$. %:
%\begin{equation}
%  \de^* = \min_t \min_{\hat y} \left[
%    \dotp{\vth^*}{\vphi(x^t,y^t)}
%  - \dotp{\vth^*}{\vphi(x^t,y)} \right]
%  > 0.
%\end{equation}
Assume the loss function is upper bounded: $\ell(y_t,\hat y)
\leq A$.  Then the cumulative sum of coefficients is upper bounded by:
\begin{equation}
  \sum_{t=1}^T \sum_{y \in \cY^t} \al^t_y
  \leq
  2A \left(\frac {\norm{\vth^*}} {\de^*}\right)^2.
\end{equation}
\end{restatable}

\noindent
% In comparison to Theorem 5.1.1 of \cite{koby-thesis}, the 
% difference in the bound is a factor of two (his result has a constant
% $4$); this derives from the symmetry in a ranking problem that is
% absent in a structured prediction problem.

However, it is not enough to show that the algorithm converges: it
could converge to a useless solution! We wish to show that in the
process of learning it does not make too many errors. In particular,
we show that Algorithm~\ref{learningalgorithm} incurs a total hinge loss
bounded by a constant also independent of $T$, which implies that at
some point it has exactly solved the learning problem.

\begin{restatable}[Total Loss]{thm}{totalloss} \label{thm:thl}
  Under the same assumptions as above, assume further that the
  norm of the examples are bounded by $R$.  Then,
  the cumulative hinge loss ($\mathcal{H}_{\delta_t}$) suffered by the algorithm
  over $T$ trials is bounded by:
  \begin{equation}
    \sum_{t=1}^T \mathcal{H}_{\delta_t}(\theta_t,(x_t, y_t))
    \leq
    8A \left( \frac {R \norm{\vth^*}} {\de^*}\right)^2.
  \end{equation}
\end{restatable}

%%% Local Variables: 
%%% mode: latex
%%% TeX-master: "main"
%%% End: 

% Acknowledgments---Will not appear in anonymized version
%\acks{We thank a bunch of people.}

\bibliography{bibfile}
\bibliographystyle{splncs}

\newpage
\appendix
\section{Minimum Cost Maximum Flow}
\label{sec:mincostflow}
Certain problems in structured prediction can be solved using inverse minimum cost maximum flow problems. Consider, for example, the task of assigning reviewers to papers. Suppose we would like to have $R$ reviewers per paper, and at most $P$ papers assigned per reviewer. Such a scenario can be easily modelled as a generalization of bipartite matching with non-unit supply and demands using a minimum cost maximum flow problem. In order to learn the weights of such an instance (suitability of a reviewer for a paper), we can use structured prediction methods that require solving inverse minimum cost maximum flow problems. Formally,

\begin{definition}[Inverse Min Cost Max Flow Problem]
Given a directed graph $G = (V,A)$, a capacity function on edges $c : E \rightarrow \mathbb{R}^+$, a cost function $w : E \rightarrow \mathbb{R}$, and a feasible maximum flow $f : E \rightarrow \mathbb{R}^+$, the inverse minimum cost maximum flow problem is to find a new cost function $w'$ such that $||w - w'||_2$ ($L_2$ norm) is minimized and $f$ is $\delta$-optimal under $w'$.
\end{definition}

Using Lemma \ref{lemma:delta-cycles}, we can easily formulate the inverse min cost max flow problem as a convex optimization problem if we can characterize a minimum cost flow in terms of small cycles in an auxiliary graph. Indeed, the following lemma is well-known.

\hide{
\begin{lemma}
Given an instance of a minimum cost maximum flow problem $G = (V,E,c,w)$, and a maximum flow $f$, $f$ is of minimum cost if and only if the residual graph $G_f$ has no negative weight cycles.
\end{lemma}
}

\begin{lemma}
Given an instance of a minimum cost maximum flow problem $G = (V,E,c,w)$, and a maximum flow $f$, $f$ is $\delta$-optimal if and only if the residual graph $G_f$ has no cycle having weight less than $\delta$.
\end{lemma}
\section{Shortest Path Trees}
\label{sec:sptrees}

Suppose we are given a directed graph $G = (V,E)$ with a weight function $w$ on edges, a subtree $T_{spt}$ rooted at $r$, and a margin $\delta$. The inverse shortest path tree problem is then to minimally modify the edge weights so that $T_{spt}$ becomes the $\delta$-optimal shortest path tree, i.e., for every vertex $v (\neq r)$ in $G$, the path prescribed by $T_{spt}$ is the $\delta$-optimal shortest path from $r$ to $v$. 

To solve this problem we define variables $d_v$ representing distance labels for each vertex and generate the following quadratic program. 
\begin{align}
  & 
\min_{w'} \sum_{e \in E}   (w'(e) - w(e))^2&&
\quad\textbf{subj. to:}\quad \\ 
 & d_r = 0 && \label{eq:dr}\\
 & d_a + w'(a,b) = d_b \quad
  &\forall e &= (a,b) \in T_{spt} \label{eq:dtspt}\\
 & d_a + w'(a,b) \geq d_b + \delta \quad
 &\forall e&=(a,b) \notin T_{spt} \label{eq:dntspt}
\end{align}
In other words, our claim is that we can require the distance labels to be the length of the unique path in the shortest path tree. For every other edge $e = (a,b)$, we would like the path to $b$ coming via $a$ to be longer by at least $\delta$. We now prove that these conditions are necessary and sufficient. 
\begin{proof}
  {\bf Sufficient:} Suppose we find a solution to the convex program. The length of the path from $r$ to $v$ via $T_{spt}$ is exactly $d_v$ due to the equality constraints. Consider any other path $P'(r,v) = [r=v_0,v_1, \ldots, v_k = v]$ from $r$ to $v$. The length of this path is exactly $\sum_{i=0}^{k-1} w'(v_i, v_{i+1}) \geq \sum_{i=0}^{k-1} (d_{v_{i+1}} - d_{v_i} + \delta')$ where $\delta' = 0$ if the edge $(v_i, v_{i+1})$ belongs to $T_{spt}$ and $\delta$ otherwise. Adding, we see that this is at least $d_v + \delta$ since at least one edge on $P'$ does not belong to $T_{spt}$.

 {\bf Necessary:} Let $w'$ be a feasible weight function, i.e., $T_{spt}$ is $\delta$-optimal under $w'$. Let $d_v$ denote the distance of vertex $v$ from $r$ under weights $w'$. Since $T_{spt}$ forms a shortest path tree, its easy to see that constraints (\ref{eq:dr}) and (\ref{eq:dtspt}) are satisfied. Now suppose that an edge $e' = (a,b) \notin T_{spt}$ does not satisfy constraint (\ref{eq:dntspt}). Then we have $d_b + \delta > d_a + w'(a,b)$. Since there is a path from $r$ to $a$ of length $d_a$, we now have a path to $b$ of length shorter than $d_b + \delta$ which violates $\delta$-optimality. Hence, we have a contradiction. %\qed
\end{proof}

%\appendix

\section{Proofs of Learning Theory Results}
\label{sec:proofs-learn-theory}

\convergence*
% \begin{theorem}[Convergence]
%   Let $\{ (x^t, y^t) \}_{t=1}^T$ be a sequence of structured examples.
% Let $\vth^*$ be any vector that separates the data with a positive
% margin $\ga^* > 0$:
% \begin{equation}
%   \ga^* = \min_t \min_{\hat y} \left[
%     \dotp{\vth^*}{\vphi(x^t,y^t)}
%   - \dotp{\vth^*}{\vphi(x^t,y)} \right]
%   > 0
% \end{equation}
% Assume the loss function is upper bounded: $\ell(\vth^t, (x^t,y^t))
% \leq A$.  Then the cumulative sum of coefficients is upper bounded by:
% \begin{equation}
%   \sum_{t=1}^T \sum_{y \in \cY^t} \al^t_y
%   \leq
%   2A \left(\frac {\norm{\vth^*}} {\ga^*}\right)^2
% \end{equation}
% \end{theorem}

\begin{proof}
Let $\cL^t(\vec\al)$ denote the lagrangian dual of the optimization problem solved in the update step of Algorithm \ref{learningalgorithm}. We have
\begin{align}
\cL^t(\vec\al) = -\frac{1}{2} ||\sum_{y \in \mathcal{Y}^t} \alpha^t_y \phi^t_{\Delta y}||^2 + \sum_{y \in \mathcal{Y}^t} \alpha^t_y(\delta_t - \theta_t \phi^t_{\Delta y}) \label{eq:lagrangian}
\end{align}
Here $\phi^t_{\De_y} = \phi(x_t,y_t) - \phi(x_t,y)$ is shorthand for the difference in feature vectors, while $\delta_t$ is the specified margin which is taken to be the current loss incurred.\\
Define $\De_t = \norm{\vth_t - \vth^*}^2 - \norm{\vth_{t+1}-\vth^*}^2$. We will establish a bound on the cumulative sum of the dual coeffecients by bounding the sum of $\De_t$s above and below.

\noindent Upper bounding:
\begin{align}
\sum_{t=1}^T \De_t &= \norm{\vth_1-\vth^*}^2 - \norm{\vth_{T+1}-\vth^*}^2  \\
&= \norm{\vth^*}^2 - \norm{\vth_{T+1}-\vth^*}^2 \label{eq:zerotheta1}\\
&\leq \norm{\vth^*}^2 \label{eq:upperbound}
\end{align}
Equation \ref{eq:zerotheta1} is obtained by substituting $\vth_1 = 0$

% $\sum_{t=1}^T \De_t = \norm{\vth_1-\vth^*}^2 - \norm{\vth_{T+1}-\vth^*}^2 = \norm{\vth^*}^2 - \textrm{const} \geq \norm{\vth^*}^2$.

\noindent Lower bounding:
\begin{align}
\De_t
&= \norm{\vth_t - \vth^*}^2 - \norm{\vth_{t+1}-\vth^*}^2 \\
&= \norm{\vth_t - \vth^*}^2
 - \norm{\vth_t + \sum_y \al_y^t \vphi_{\De y}^t - \vth^*}^2 \\
&= \norm{\vth_t - \vth^*}^2
 - \Bigg[ \norm{\vth_t - \vth^*}^2
        + \norm{\sum_y \al_y^t \vphi_{\De y}^t}^2\nonumber\\
&\qquad\qquad\qquad\qquad
        + 2\dotp{(\vth_t - \vth^*)}{\sum_y \al_y^t \vphi_{\De y}^t}
        \Bigg] \\
&= - \norm{\sum_y \al_y^t \vphi_{\De y}^t}^2
   - 2 \dotp{\vth_t}{\sum_y \al_y^t \vphi_{\De y}^t} %\nonumber\\
%&\qquad\qquad\qquad\qquad
   + 2 \dotp{\vth^*}{\sum_y \al_y^t \vphi_{\De y}^t} \\
\intertext{Substituting for $\cL^t(\vec\al)$ from Eq. \eqref{eq:lagrangian}, we get}
\De_t&= 2 \left[ 
         \cL^t(\vec\al)
       + {\sum_y \dotp{\vth^*}{\al_y^t \vphi_{\De y}^t}}
       - \sum_y \al_y^t \delta_t \right] \\
\intertext{As $\vec\al = 0$ is dual feasible and $\cL^t(0) = 0$, we have $\cL^t(\vec\al) \geq 0$}       
&\geq 2 \left[
         \sum_y \dotp{\vth^*}{ \al_y^t \vphi_{\De y}^t}
       - \sum_y \al_y^t \delta_t \right] \\
&\geq 2 \left[
         \sum_y \al_y^t \delta^*
       - \sum_y \al_y^t \delta_t \right] \\       
&\geq 2 \sum_y \al_y^t (\delta^* - \delta_t) \\
&\geq 2 (\delta^* - A) \sum_y \al_y^t
\intertext{In the last step, we used the bound on the instantaneous loss.  In the
previous, we used the assumption on the margin achieved by $\vth^*$.
The rest is algebra.}
\intertext{Summing over all $t$ we get}
\sum_{t=1}^T \De_t &\geq 2(\delta^* - A) \sum_{t=1}^T \sum_{y} \alpha^t_y \label{eq:lowerbound}
\intertext{Combining the bounds in Equations \eqref{eq:upperbound} and \eqref{eq:lowerbound}}
2 (\delta^* - A) \sum_t \sum_y \al^t_y &\leq \norm{\vth^*}^2
\end{align}
Now, fix $c = \frac {2A} {\delta^*}$ and scale $\vth^*$ and $\delta^*$ by $c$. 
Rearrange to get the desired bound.
\end{proof}

\begin{lemma}
Under the same assumptions as before, and writing $\vec\Phi^t_{\De \circledcirc}$ to denote a $\card{\cY^t} \times F$ matrix whose rows are the feature vectors for all possible outputs, and where $p$ and $q$ are dual (i.e., $\frac 1p+\frac 1q=1$), the optimal dual variables $\al^t_y$ satisfy:
\begin{equation}
\delta_t - \dotp{\vth_t}{\vphi^t_{\De_y}} \leq \norm{ \vec\al^t }_p \norm{ \vec\Phi^t_{\De \circledcirc} \vphi^t_{\De y} }_q
\end{equation}
\end{lemma}
\begin{proof}
By the enforced $\delta_t$ - optimality conditions, we know that for all $t$ and $y \in \mathcal{Y}^t$ - 
\begin{align}
\dotp{\vth_{t+1}}{\vphi^t_{\De_y}} &\geq \delta_t\\
\intertext{Substituting for $\vth_{t+1}$ in terms of $\vth_t$ using the dual optimality conditions}
 \dotp{(\vth_t + \sum_{z \in \cY^t} \al^t_z \vphi^t_{\De_z} )}{\vphi^t_{\De y}} &\leq \delta_t\\
\delta_t - \dotp{\vth_t}{\vphi^t_{\De_y}}
&\leq
\sum_{z \in \cY^t} \al^t_z \dotp{\vphi^t_{\De z}}{\vphi^t_{\De_y}} \\
&=
\dotp{\vec\al^t}{ \vec\Phi^t_{\De \circledcirc} \vphi^t_{\De_y} } \\
&\leq
\norm{\vec\al^t}_p  \times \norm{ \vec\Phi^t_{\De \circledcirc} \vphi^t_{\De_y} }_q\label{eq:margin-size-bound}
\end{align}
The first equality is rewriting things in terms of the matrix
$\vec\Phi$, the final step is H\"older's inequality.  
\end{proof}

\totalloss*
% \begin{theorem}[Total Hinge Loss]
%   Under the same assumptions as before, assume further that the
%   examples are bounded (i.e., $\norm{\vphi(x^t,y^t)} \leq R$).  Then,
%   the cumulative hinge loss $\ell_{\ga_t}$ suffered by the algorithm
%   over $T$ trials is bounded by:
%   \begin{equation}
%     \sum_{t=1}^T \ell_{\ga_t}(\vth^t, (x^t,y^t))
%     \leq
%     8A \left( \frac {R \norm{\vth^*}} {\ga^*}\right)^2
%   \end{equation}
% \end{theorem}
\begin{proof}
  On a round $t$ with non-zero hinge loss, take the $y$ in
  Eq~\eqref{eq:margin-size-bound} that has maximal hinge loss
  (numerator).  Take $p=1$ and $q=\infty$.  Then:
  \begin{align}
\sum_t H_{\delta_t}(\vec \theta_t, (x_t,y_t)) &\leq \sum_t (\delta_t - \dotp{\vec \theta_t}{\phi^t_{\De_y}})\\
&\leq \sum_t \norm{\vec\al^t}_1 \norm{ \vec\Phi^t_{\De \circledcirc} \vphi^t_{\De y} }_\infty \\
    &= \sum_t (\sum_z \ab{\al^t_z}) (\max_z \dotp{\vphi^t_{\De z}}{\vphi^t_{\De y}}) \\
\intertext{Using H\"older's inequality again with $p = q = 2$}
    &\leq \sum_t (\sum_z \ab{\al^t_z}) (\max_z \norm{\vphi^t_{\De z}}_2 \norm{\vphi^t_{\De y}}_2) \\
\intertext{By assumption the norm is bounded,}
&\leq \sum_t (\sum_z \ab{\al^t_z}) (\max_z (2R) (2R)) \\
    &\leq 4 R^2 \sum_t \sum_z \ab{\al^t_z} \\
\intertext{Bounding the cumulative sum using Theorem \ref{thm:convergence}, we get}
    &\leq 4 R^2 2A \left(\frac {\norm{\vth^*}} {\delta^*}\right)^2 \\
    &= 8A \left( \frac {R \norm{\vth^*}} {\delta^*}\right)^2
  \end{align}
\end{proof}

\section{Experimental Analysis}
\label{sec:experiments}

We perform preliminary experiments to demonstrate the efficacy of the online learning framework. We consider two structured prediction tasks: dependency parsing and word alignment for language translation.
In dependency parsing, the input is a sentence, and the goal is to find its dependence parse, i.e. evaluate how words in a sentence relate to one-another, forming a tree starting with an empty root node. Figure \ref{fig:dependency-parse} shows an example of a dependency parse tree of a sentence. The input sentence is considered as a complete graph with a vertex for each word and each edge parameterized by features. The task now is to learn feature weights so that the maximum weight spanning tree in the graph corresponds to the parse tree of the sentence. In word alignment for language translation, we are given two equivalent sentences in two languages and the task is to identify corresponding words. The input instance in this case is considered to be a complete bipartite graph, and the output would be an assignment (matching). Once again given features on edges, the task is to learn feature weights so that the maximum weight perfect matching in the graph would correspond to the correct word alignment. 
\begin{figure}
  \centering
  \includegraphics[width=0.6\textwidth]{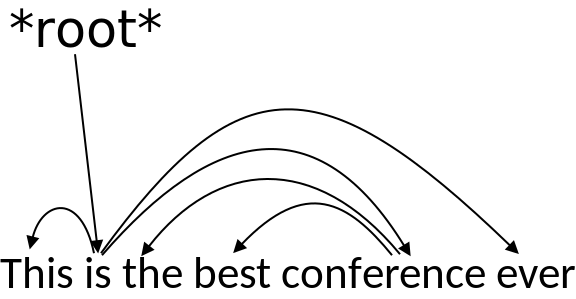}
\caption{Example dependency parse tree. The tree describes the relations between head words and their dependents in the sentence}
  \label{fig:dependency-parse}
\end{figure}
% maximum spanning trees for dependency parsing and bipartite matching for word alignment.

\subsection{Maximum Spanning Trees}
\label{sec:mstexp}
Although dependency parsing is better modelled by directed arborescences, for the sake of simplicity we consider only spanning trees in directed graphs in our experiments.
For these experiments we used the CoNLL shared task English treebank
\cite{conll07} in order to predict undirected dependency arcs in English
sentences. %  This task evaluates how words in a sentence relate to
% one-another, forming a tree starting with an empty \emph{root} node.
Each word only depends on one word, but can have many dependents.  We
use a $1500$-sentence subset of the training data ($36k$ words) and
the test data consists of $3800$ sentences ($90k$ words).  We train
for undirected unlabeled dependencies and evaluate in the same manner.
We use standard features: words, word suffixes, position, edge length
and predicted part of speech tags.

An averaged structured perceptron baseline obtains an accuracy of
$82.7\%$ on the test data.  One-best MIRA achieves an accuracy of
$84.2\%$, which is very close to the performance of a structured SVM
trained by stochastic gradient descent (accuracy of $84.4\%$).  Our
approach achieves a significant improvement on this of $85.1\%$.

\begin{table}
\centering

\begin{tabular}{|c||c|c}
\hline
Algorithm & Averaged $\theta$ \\ \hline \hline 
SVM & 84.4 \\
Perceptron & 82.7 \\
1-Best MIRA & 84.2 \\
Our Algorithm & \textbf{85.1} \\
\hline

\end{tabular} \vspace{2mm}
\caption{Accuracy results for undirected dependency parsing
over the various baselines and our algorithm.}
\label{tab:deptable}
\end{table}

\subsection{Bipartite Matching} \label{sec:matching2}

We are considering German/English alignment at the word level. Given two equivalent sentences in their respective languages, we want to choose an alignment that best fits the ``equality'' of the respective words in each sentence.

Our data is comprised of 217 manually word-aligned sentences from \cite{CCB05para}, with many-to-many matchings possible. Since we are restricted to one-to-one alignment (matching) ,we enforce this restriction here by pruning extra edges until we have 1-to-1 matchings only.   

The graph structure is as follows: Given a German sentence of length $n$ and an English sentence of length $m$, we construct a complete bipartite graph $G=(X_{m}, Y_{n}, E)$, where $X_{m}$ is the German sentence of length $m$. $Y_{n}$ is an English sentence of length $n$. $E$ are the possible alignments for the two sentences, which are currently fully connected. Further, to ensure that the induced matching will be perfect(as required from our problem definition), we add $n$ extra vertices(referred to as dummy vertices) to the German sentence and $m$ dummy vertices to the English sentence, resulting in $G'=(X_{m} \cup X'_{n}, Y_{n} \cup Y'_{m}, E')$.

Next we will describe the structure of $E'$ in the graph $G'$. Each vertex of the original $G$ is fully connected with every other vertex in $G$ as before. Each dummy vertex $x'_i \in X'_{n}$ is fully connected to each dummy node $y_j' \in Y'_{m} ,\forall j$. In addition, each dummy vertex $x'_i \in X'_{n}$ is also connected to its single, respective real word vertex $y_i \in Y_{n}$, and similarly $Y'_{m}$ to $X_{m}$. If an alignment for a particular word in $X_{m}$ or $Y_{n}$ is not present in $M_{g}$ due to the nature of word alignment being somewhat sparse, we designate the truth edge to be the one connected to its corresponding dummy vertex and add this to $M_{g} \subset E$, giving us $M'_{g} \subset E'$. This will allow us a perfect matching for any two sentences we are given, fitting this particular problem's framework.

Each edge of the graph $G'$ corresponds to a set of feature values for the corresponding words. These include 
features such as their Dice Coefficient (computed from Europarl corpus), relative word positioning 
in the sentence, string match without vowels, and others detailed in \cite{taskar05discrim}. 
We also created slack features and weights, one for each viable edge in the graph. 
For example, the feature ``$4\_10$'' means ``this edge connects node 4 to node 10''. 
These features are very small valued (the result being that making these features important 
 is expensive), and are used to ensure feasibility during training time for each example presented.
 The edges corresponding to dummy vertices will \emph{only} have these slack features, 
which will make certain this perfect matching problem is feasible. 
After each example is learned, we immediately forget these slack weights.%, as we do not want to learn arbitrary weights.
%For our algorithm we omitted the large-margin extension due to the $O(n^3)$ constraints required, which made larger sentences intractable with the current implementation.

\subsubsection{Results}

\begin{table}[h]
\centering
\begin{tabular}{|c||c|c}
\hline
Algorithm & Averaged $\theta$ \\ \hline \hline 
Dice=1  & 36.40 \\
Perceptron & 40.52 \\
1-Best MIRA & 13.76 \\
Struct SVM & 30.42 \\
Our Algorithm & \textbf{44.00} \\
\hline
\end{tabular} \vspace{2mm}
\caption{Accuracy results for 1-to-1 matching with 10 passes over data with 80-20 train-test split}% using 10 passes over the dataset, using a random 80-20 train-test split. The results were then averaged over 10 total runs.}
\label{tab:matchtable}
\end{table}

One simple baseline we are testing against is the use of only DICE values. This is being used as a sanity-check for other algorithms' performance. This simple baseline obtained an accuracy of 36.4\%.
Averaged MIRA performs quite poorly on this dataset with these features, at 13.76\%. In our experiments we are only comparing our algorithm with 1-best MIRA, which is unfair as it only constructs the constraints based on the single best matching based on our old weight vector w. Limited to optimizing against one matching appears to be insufficient to learn anything of value. Again however, our algorithm supersedes any $k$ chosen for MIRA, as we ensure our matching beats every single other possible matching.
Averaged Perceptron performs better on this dataset, at 40.52\%. The variance from run to run was very small compared to the other baselines.
Our algorithm obtains an accuracy of 44.0\%.
This failure of 1-Best MIRA is analogous to a corresponding failure of
structured SVMs optimized with cutting plane(30.42\%)
\cite{tsochantaridis05svmiso,joachims09cuttingplane}: when the problem
is ``hard'' (in the sense that $\vth^*$ has high loss), this approach
appears to perform quite poorly in practice \cite{papandreou11perturb,tarlow12rom}.

%%% Local Variables: 
%%% mode: latex
%%% TeX-master: "main"
%%% End: 

\end{document}